\documentclass{article} 
\usepackage{wise_preprint,times}

\usepackage{amsmath,amsfonts,bm}

\def\eqref#1{equation~\ref{#1}}

\def\1{\bm{1}}

\DeclareMathAlphabet{\mathsfit}{\encodingdefault}{\sfdefault}{m}{sl}
\SetMathAlphabet{\mathsfit}{bold}{\encodingdefault}{\sfdefault}{bx}{n}

\usepackage{bbding}
\usepackage{pifont}

\usepackage[utf8]{inputenc} 
\usepackage[T1]{fontenc}    
\usepackage{hyperref}       
\usepackage{url}            
\usepackage{amsfonts}       
\usepackage{nicefrac}       
\usepackage{algorithmicx}   
\usepackage{microtype}      
\usepackage{xcolor}         
\usepackage{siunitx}
\definecolor{lightcrimson}{rgb}{0.93, 0.16, 0.51}

\usepackage{amsmath, amssymb, amsthm}

\usepackage{booktabs}
\usepackage{threeparttable}
\usepackage{graphicx, subcaption, adjustbox, wrapfig, multirow, array, tabularx, enumitem, ulem,colortbl}
\usepackage{xcolor, hyperref}
\usepackage[table,xcdraw]{xcolor}
\usepackage{tikz}
\usepackage{algorithm}
\usepackage{algpseudocode}
\PassOptionsToPackage{numbers, compress}{natbib} 
\usepackage{amsthm}


\theoremstyle{plain}

\title{Attention Routing Stabilizes Early: Working-Set Inference for Recurrent-Depth Language Models}

\author{
Ke Wan \\
Department of Computer Science \\
University of Virginia \\
Charlottesville, VA, USA \\
\texttt{tbn5pj@virginia.edu}
\And
Chen Chen\thanks{Corresponding author.} \\
Department of Computer Science \\
University of Virginia \\
Charlottesville, VA, USA \\
\texttt{zrh6du@virginia.edu}
}

\iclrfinalcopy 
\begin{document}

\maketitle

\begin{abstract}

Recurrent-depth language models, such as looped Transformers, repeatedly apply shared network blocks to refine latent representations, enabling additional test-time computation without generating explicit intermediate reasoning tokens.
However, recurrent inference repeatedly accesses long contexts: full attention incurs quadratic computation during prefill, while autoregressive decoding repeatedly reads large key-value (KV) caches, introducing substantial memory traffic and capacity pressure.
We study how attention routing evolves across recurrent depth and uncover a consistent separation in convergence timescales: routing-related quantities, including the attention support and the attention distribution, stabilize substantially earlier than representation-related quantities such as hidden states and attention outputs.
This suggests a two-stage structure in recurrent inference, where the model first discovers a sparse working set of relevant context and then continues refining representations over largely the same routing support.
Motivated by this structure, we introduce \texttt{WISE} (\textbf{W}orking-set \textbf{I}nference with \textbf{S}upport \textbf{E}xploitation), a training-free method that uses unrestricted global attention during early recurrent steps to discover a block-structured working set, and then reuses it as the routing support in later steps while keeping recurrent refinement and within-support attention computation dynamic.
By combining sparse recurrent attention with compact KV-cache management, \texttt{WISE} reduces both redundant attention computation and memory movement during inference.
Across reasoning and long-context benchmarks, \texttt{WISE} closely approaches full-attention accuracy on key evaluated workloads while outperforming early-static support baselines.
At 64K context, our implementation achieves up to a $6.46\times$ decode throughput improvement under GPU memory pressure.
Our code is available at \url{https://github.com/tbn5pj/WISE_code}.

\end{abstract}

\section{Introduction}

Recurrent-depth language models provide a new axis for test-time computation by repeatedly applying shared network blocks to refine latent representations \citep{NEURIPS2025_3b01972c,dehghani2018universal,yang2024looped,giannou2023looped,rodkin-etal-2026-beyond}, rather than generating explicit intermediate reasoning tokens \citep{NEURIPS2022_8bb0d291}.
This decouples parameter count from inference-time compute, allowing additional computation to be allocated through recurrent depth rather than additional parameters \citep{dehghani2018universal,graves2016adaptive}.
Prior work has shown that additional recurrent computation can improve multi-step reasoning \citep{NEURIPS2025_3b01972c}.
However, existing recurrent inference recomputes full self-attention and restores the full key-value cache over the entire context at every recurrent step, repeating global attention routing (i.e., exchanging information across all token positions), which is quite computation- and memory-heavy.
It remains unclear whether later steps still need to search the full context, or whether they mainly continue refining representations over context that earlier steps have already identified as relevant.
Rather than asking how much recurrent computation to use, we take this budget as given and ask: \emph{must every recurrent step continue to perform unrestricted global attention routing?}

\begin{figure*}[t]
\centering

\begin{minipage}[c]{0.34\textwidth}
\centering
\includegraphics[
    height=4.15cm,
    keepaspectratio
]{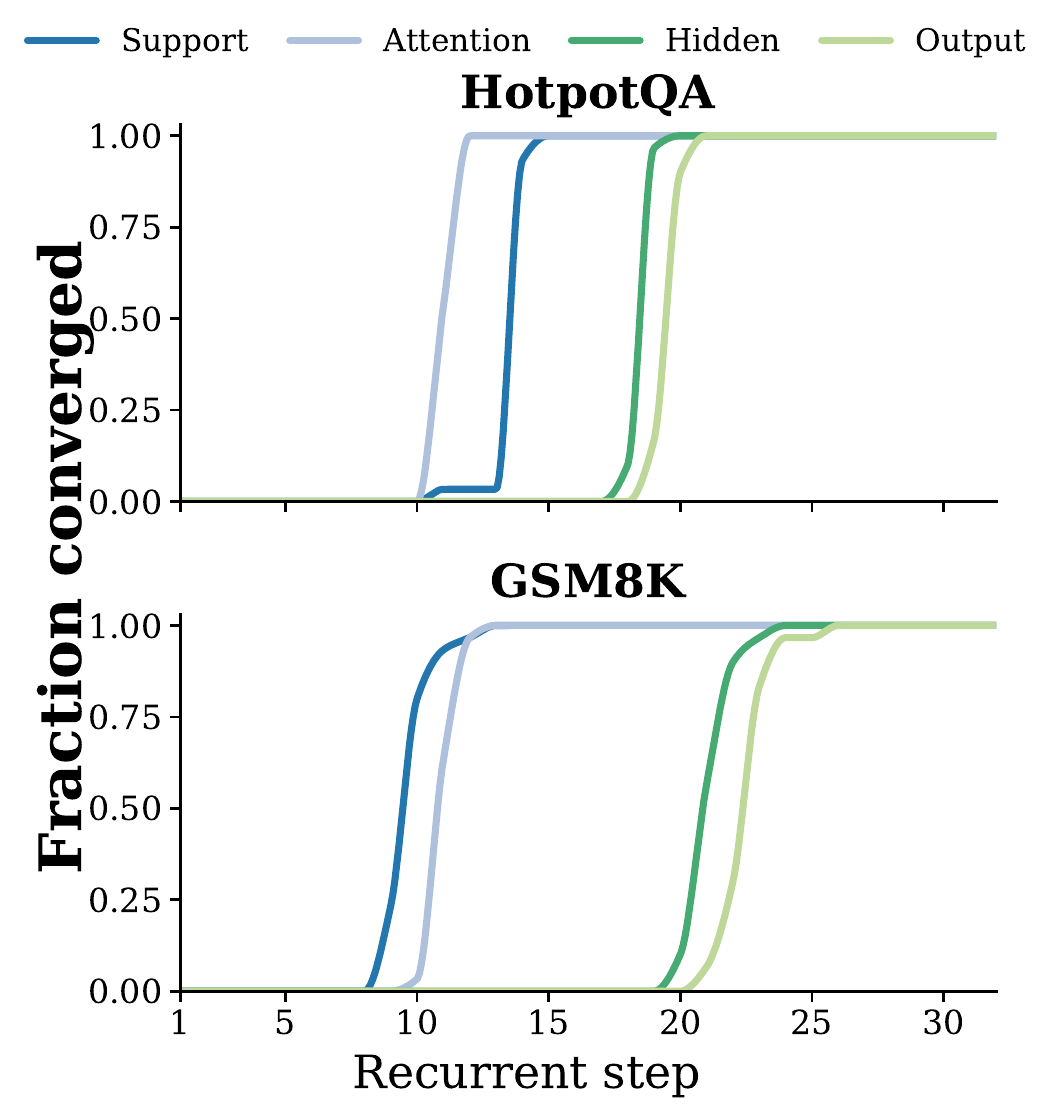}
\phantomsubcaption
\label{fig:mechanism}
\par
\vspace{0.5mm}
\textbf{(\subref{fig:mechanism}) Convergence timescales}
\end{minipage}
\hspace{-0.005\textwidth}
\begin{minipage}[c]{0.58\textwidth}
\centering
\includegraphics[
    height=4.15cm,
    keepaspectratio
]{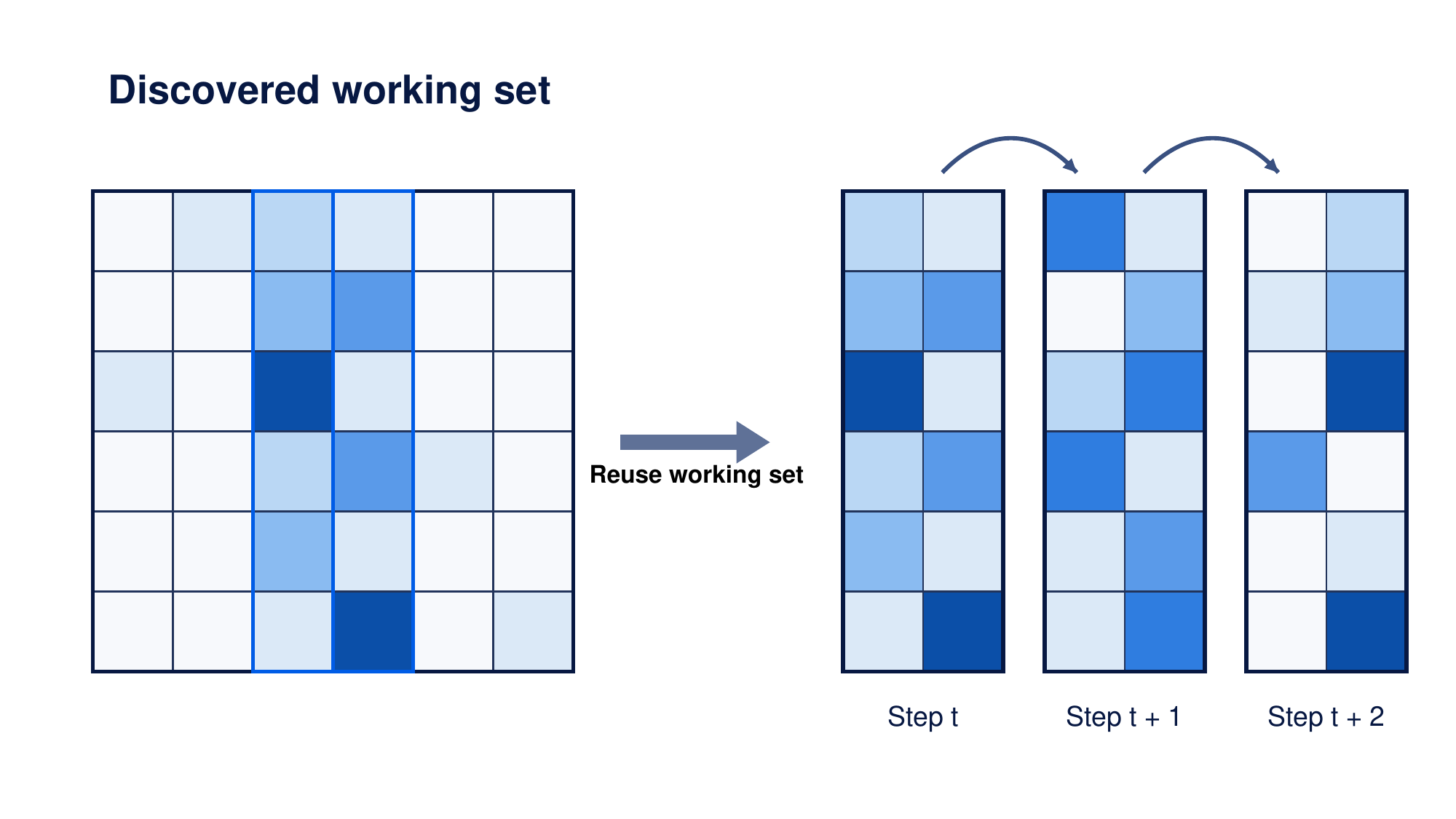}
\phantomsubcaption
\label{fig:method}
\par
\vspace{0.5mm}
\textbf{(\subref{fig:method}) Working-set reuse}
\end{minipage}

\caption{
\textbf{Routing stabilizes before representation refinement is complete.}
\textbf{(a)} Fraction of examples converged at each recurrent step for Huginn, on 30 benchmark-native examples from each of HotpotQA and GSM8K.
Routing-related quantities (diagnostic token-level attention support and the attention distribution) stabilize substantially earlier than representation-related quantities (hidden states and attention outputs).
Whether the attention support or the attention distribution stabilizes first varies across tasks, but the separation between routing and representation is consistent across both tasks.
\textbf{(b)} This separation motivates \texttt{WISE}: early recurrent steps use unrestricted global attention to discover a sparse block-structured working set, and later steps reuse only this routing support, while hidden representations, queries, keys, values, and within-support attention weights continue to evolve.
Thus, \texttt{WISE} reuses \emph{where} attention may route without freezing \emph{how} information within the working set is processed: \emph{discover early, reuse late}.
}
\label{fig:wise_motivation}

\end{figure*}

To answer this question, we study how attention routing evolves across recurrent depth and uncover a consistent separation in convergence timescales: routing-related quantities, which describe \emph{where} each token attends (the set of keys receiving most attention mass, i.e., the attention support, and the attention distribution over keys), stabilize substantially earlier than representation-related quantities, which describe \emph{how} the retrieved information is processed (hidden states and attention outputs).
Whether the attention support or the attention distribution stabilizes first varies across tasks and models, but routing consistently stabilizes before representation across benchmarks, a range of convergence thresholds, and two different recurrent-depth models.
As illustrated in Figure~\ref{fig:mechanism}, recurrent models can settle on a stable region of information access while their representations are still evolving.
In other words, the model largely determines \emph{where to look} before it finishes refining \emph{how to use} the retrieved information.

This separation suggests a two-stage structure in recurrent inference: early steps perform global search to discover a sparse \emph{working set} of relevant context, while later steps continue refining representations over largely the same routing support.
Standard recurrent inference does not exploit this structure, yet if late recurrent steps attend only to a working set of $m \ll n$ keys per query, where $n$ is the context length, their attention cost can in principle drop from $O(n^2)$ to $O(nm)$ during prefill without reducing recurrent depth.
During decoding, the same sparsity can reduce the storage complexity of the selected recurrent KV states from $O(n)$ to $O(m)$ and similarly reduce memory accesses, alleviating GPU memory pressure.

Motivated by this structure, we introduce \texttt{WISE} (\textbf{W}orking-set \textbf{I}nference with \textbf{S}upport \textbf{E}xploitation), a training-free inference method illustrated in Figure~\ref{fig:method}.
\texttt{WISE} uses unrestricted global attention during an early discovery phase to identify the working set, and then keeps it fixed as the routing support for the remaining recurrent steps.
During prefill, the working set is organized into contiguous token blocks, enabling efficient block-sparse attention in later recurrent steps.
During decoding, only the KV states corresponding to the discovered working set need to be retained for these later steps, reducing persistent KV-cache storage and memory accesses.
Crucially, \texttt{WISE} fixes only \emph{where} attention may route: hidden representations, queries, keys, values, and attention weights within the working set continue to be recomputed at every step.
Thus, \texttt{WISE} preserves late representation refinement while avoiding repeated global attention over positions outside the discovered working set.

Across reasoning and long-context benchmarks, \texttt{WISE} closely approaches full-attention accuracy on key evaluated workloads while outperforming early-static support baselines.
At 64K context on an NVIDIA RTX A6000, our implementation achieves up to a $6.46\times$ decode-only throughput improvement and a $2.58\times$ end-to-end serving throughput improvement over lossless full-attention inference under GPU memory pressure.
Together, these results show that recurrent routing structure can become computationally reusable, reducing both attention computation and KV-cache overhead while preserving late representation refinement.

\paragraph{Contributions.}
We summarize our contributions as follows:
(i) \textbf{Mechanism discovery.}
We identify a consistent separation between routing and representation timescales in recurrent-depth language models: the attention support and the attention distribution stabilize substantially earlier than hidden representations and attention outputs, suggesting that information-access structure can become reusable before recurrent refinement is complete.
(ii) \textbf{Method.}
We introduce \texttt{WISE}, a training-free inference method that discovers a sparse block-structured working set using unrestricted attention in early recurrent steps and reuses it as the routing support in later steps.
\texttt{WISE} preserves recurrent depth and dynamic within-support computation while enabling block-sparse attention and compact KV-cache management.
(iii) \textbf{Empirical validation and efficiency.}
We validate support-only reuse through controlled interventions and evaluate its quality--efficiency tradeoff on reasoning and long-context benchmarks.
Our implementation reduces KV-cache memory requirements and achieves up to $6.46\times$ decode-only throughput and $2.58\times$ end-to-end serving throughput improvements at 64K context under GPU memory pressure.

\section{Related Work}

\textbf{Recurrent and Latent Test-Time Computation.}
Recurrent architectures increase effective computational depth by repeatedly applying shared network blocks, from \textbf{Universal Transformers} to more recent looped Transformers for iterative computation and reasoning \citep{dehghani2018universal,giannou2023looped,yang2024looped,saunshi2025reasoning}. Recent work extends this idea to language-model inference, using recurrent depth to allocate additional test-time computation and improve multi-step reasoning \citep{NEURIPS2025_3b01972c,rodkin-etal-2026-beyond,zhang2026modr,kohli2026loop}. Related approaches reason directly in continuous latent states \citep{hao2024training}. These works establish recurrent refinement as a useful computational mechanism. \texttt{WISE} instead asks whether every recurrent step must continue to perform unrestricted global attention routing, and studies how routing and representation refinement evolve across recurrent depth.

\textbf{Adaptive Transformer Computation.}
A complementary line of work reduces inference cost by adapting how much computation different inputs or tokens receive. \textbf{Adaptive Computation Time}, \textbf{Mixture-of-Depths}, and early-exit methods dynamically reduce recurrent or layer depth \citep{graves2016adaptive,banino2021pondernet,elbayad2019depth,raposo2024mixture,xin2020deebert,zhou2020bert,elhoushi-etal-2024-layerskip}. \texttt{WISE} is orthogonal: it preserves recurrent depth and reduces only the scope of late global routing, while keeping attention weights and hidden-state updates dynamic.


\textbf{Efficient and Sparse Attention.}
Kernel-level methods such as \textbf{FlashAttention} and \textbf{FlashAttention-2} accelerate exact dense attention through IO-aware tiling and improved parallelism \citep{dao2022flashattention,ICLR2024_98ed250b}. Long-context methods further reduce attention cost through KV-cache compression or sparse retrieval, including \textbf{H$_2$O}, \textbf{SnapKV}, \textbf{Quest}, and \textbf{MInference} \citep{NEURIPS2023_6ceefa7b,li2024snapkv,tang2024quest,NEURIPS2024_5dfbe6f5}. Other work reuses attention patterns across layers or decoding steps \citep{xiao2019sharing,xu-etal-2025-refreshkv,deshmukh2025kascade}. \texttt{WISE} instead targets recurrent-depth language models and reuses only the routing support after it stabilizes, while within-support attention and representations remain dynamic; this support-level sparsity complements efficient kernels such as FlashAttention.

\section{Problem Setup}
\label{sec:problem_setup}

\paragraph{Recurrent inference.}
Consider a recurrent-depth language model that repeatedly applies a shared transformation to a sequence of hidden representations. Given a sequence of length $n$, let $\boldsymbol{H}^{(t)} \in \mathbb{R}^{n \times d}$ denote the hidden states after recurrent step $t$, where $d$ is the hidden dimension. The recurrent computation is written as
\begin{equation}
    \boldsymbol{H}^{(t+1)}
    =
    F_{\theta}\!\left(\boldsymbol{H}^{(t)}\right),
    \qquad
    t=0,\ldots,T-1,
\label{eq:recurrent_dynamics}
\end{equation}
where the parameters $\theta$ are shared across recurrent steps and $T$ denotes the recurrent depth. This formulation abstracts away model-specific non-recurrent components and focuses on the computation repeatedly applied within the recurrent core.

\paragraph{Recurrent attention.}
At each recurrent step, self-attention is recomputed from the current hidden representations. Suppressing recurrent-block and attention-head indices for clarity, we define $\boldsymbol{Q}^{(t)}=\boldsymbol{H}^{(t)}\boldsymbol{W}_{Q}$, $\boldsymbol{K}^{(t)}=\boldsymbol{H}^{(t)}\boldsymbol{W}_{K}$, and $\boldsymbol{V}^{(t)}=\boldsymbol{H}^{(t)}\boldsymbol{W}_{V}$, where $\boldsymbol{Q}^{(t)},\boldsymbol{K}^{(t)},\boldsymbol{V}^{(t)}\in\mathbb{R}^{n\times d_h}$ and $d_h$ is the attention-head dimension. The corresponding attention matrix and attention output are
\begin{equation}
    \boldsymbol{A}^{(t)}
    =
    \operatorname{softmax}\!\left(
        \frac{\boldsymbol{Q}^{(t)}\boldsymbol{K}^{(t)\top}}{\sqrt{d_h}}
        +\boldsymbol{M}
    \right),
    \qquad
    \boldsymbol{O}^{(t)}
    =
    \boldsymbol{A}^{(t)}\boldsymbol{V}^{(t)},
\label{eq:recurrent_attention}
\end{equation}
where $\boldsymbol{M}$ denotes the attention mask. For query position $i$, we write $\boldsymbol{a}_{i}^{(t)}=\boldsymbol{A}^{(t)}_{i,:}$ for its attention distribution and let $\mathcal{K}_{i}\subseteq\{1,\ldots,n\}$ denote the keys admissible under $\boldsymbol{M}$.
During autoregressive decoding, attention is computed for newly generated queries using cached historical keys and values, rather than recomputing full-sequence attention.

\paragraph{Attention support and working sets.}
For a mass threshold $\eta\in(0,1)$, let $S_{i}^{(t,\eta)}\subseteq\mathcal{K}_{i}$ denote the smallest set of admissible keys whose cumulative attention mass under $\boldsymbol{a}_{i}^{(t)}$ is at least $\eta$; we refer to $S_{i}^{(t,\eta)}$ as the \emph{attention support}. More generally, a \emph{working set} is a sparse routing support discovered during early recurrent computation and reused at later recurrent steps. For structured execution, we partition the sequence into blocks of size $B$ and denote by $\mathcal{B}_{q}^{(t,\eta)}$ the active key-block support associated with query block $q$ at step $t$; its construction from block-level attention mass is specified in Section~\ref{sec:wise}. Unless explicitly shown, recurrent-block and attention-head indices are suppressed for both $S_{i}^{(t,\eta)}$ and $\mathcal{B}_{q}^{(t,\eta)}$.
During prefill, standard recurrent inference recomputes global attention over all admissible query--key pairs at every recurrent step, even when this routing structure has already stabilized.
During decoding, the discovered working set can additionally determine which historical KV states need to be retained for later recurrent steps, while earlier steps requiring unrestricted attention retain their full KV caches.

\section{Computational Structure of Recurrent Inference}
\label{sec:computational_structure}

The distinction between recurrent representation refinement and attention routing introduced in Section~\ref{sec:problem_setup} has direct computational and memory implications for later recurrent steps, during both prefill and autoregressive decoding.

\paragraph{Prefill computation.}
For a sequence of length $n$, linear projections and feed-forward layers scale as $O(nd^2)$, whereas full self-attention computes pairwise query--key interactions and aggregates values, resulting in $O(n^2d)$ computation across attention heads.
Thus, the computational cost of each recurrent step under full attention is
\begin{equation}
    C_{\mathrm{full}}
    =
    O(nd^2)+O(n^2d).
\label{eq:full_recurrent_cost}
\end{equation}
However, Figure~\ref{fig:wise_motivation}\subref{fig:mechanism} reveals that attention routing stabilizes substantially earlier than representation refinement.
If early recurrent computation identifies a working set containing on average $m \ll n$ relevant keys per query, later recurrent steps can restrict attention to this working set, reducing the attention cost from
\begin{equation}
    O(n^2d)
    \quad\longrightarrow\quad
    O(nmd).
\label{eq:working_set_cost}
\end{equation}
Equivalently, defining the normalized working-set density $\rho=m/n$, the late-stage attention cost becomes $O(\rho n^2d)$, while recurrent representation refinement remains unchanged.

\paragraph{Decode KV-cache storage.}
During autoregressive decoding, each new query attends to historical keys and values stored in the KV cache.
Standard recurrent inference retains the full historical KV states for each recurrent step.
By restricting later recurrent steps to the discovered working set, WISE only needs to retain the corresponding KV states, reducing the per-step storage complexity from
\begin{equation}
    O(nd)
    \quad\longrightarrow\quad
    O(md).
\label{eq:decode_kv_memory}
\end{equation}
This also reduces repeated KV-cache accesses, alleviating GPU memory pressure and potentially avoiding costly host--device transfers during long-context decoding.

Together, these computational and memory savings motivate \texttt{WISE}: \emph{discover globally early, then reuse the working set late}.


\section{\texttt{WISE}: Discover Early, Reuse Late}
\label{sec:wise}

Section~\ref{sec:computational_structure} shows that restricting later recurrent steps to a sparse working set of $m\ll n$ keys can reduce attention computation from $O(n^2d)$ to $O(nmd)$ during prefill and KV-cache storage from $O(nd)$ to $O(md)$ during decoding.
\texttt{WISE} realizes this structure by separating recurrent inference into two stages: unrestricted global attention discovers the relevant routing support early, and the discovered support is reused later while recurrent refinement and all within-support computation remain dynamic.
The support-reuse principle itself is agnostic to routing granularity, but unstructured token-level sparsity is difficult to translate directly into efficient GPU execution.
We therefore instantiate \texttt{WISE} with block-structured support, trading some fine-grained sparsity for regular, reusable computation that block-sparse kernels can execute efficiently.

\subsection{Block-Structured Working-Set Discovery}
\label{sec:wise_discovery}

During the first $t_{\mathrm d}$ recurrent steps, \texttt{WISE} retains unrestricted attention over the full context.
At the discovery step $t_{\mathrm d}$, we collect attention statistics to identify a sparse working set for subsequent reuse.
We partition the sequence into blocks of size $B$ and let $\mathcal{I}_q$ denote the token positions in query block $q$ and $\mathcal{I}_b$ those in key block $b$.
Suppressing recurrent-block and attention-head indices as in Section~\ref{sec:problem_setup}, we define the attention mass assigned from query block $q$ to key block $b$ as
\begin{equation}
    r_{qb}^{(t_{\mathrm d})}
    =
    \frac{1}{|\mathcal{I}_q|}
    \sum_{i\in\mathcal{I}_q}
    \sum_{j\in\mathcal{I}_b}
    \big[\boldsymbol{A}^{(t_{\mathrm d})}\big]_{ij}.
\label{eq:block_attention_mass}
\end{equation}
For efficient online discovery, block-level attention mass can be estimated from sampled query positions without materializing the full attention matrix.
Given a cumulative-mass threshold $\eta$, the discovered working set is the smallest set of highest-mass admissible key blocks whose cumulative attention mass reaches $\eta$:
\begin{equation}
    \mathcal{B}_{q}
    =
    \underset{\mathcal{B}}{\arg\min}\ |\mathcal{B}|
    \quad
    \mathrm{s.t.}
    \quad
    \sum_{b\in\mathcal{B}} r_{qb}^{(t_{\mathrm d})}
    \geq \eta .
\label{eq:block_working_set}
\end{equation}
Key blocks are selected in descending order of attention mass with deterministic tie-breaking.


\subsection{Working-Set Reuse}
\label{sec:wise_reuse}

For recurrent steps $t>t_{\mathrm d}$, the discovered block support $\mathcal{B}_q$ remains fixed, but the hidden representations and attention computation continue to evolve.
Let $q(i)$ and $b(j)$ denote the query and key blocks containing positions $i$ and $j$, respectively, and define the working-set mask
\begin{equation}
    \big[\boldsymbol{R}\big]_{ij}
    =
    \begin{cases}
        0, & b(j)\in\mathcal{B}_{q(i)},\\
        -\infty, & \text{otherwise}.
    \end{cases}
\end{equation}
The late-stage attention is recomputed as
\begin{equation}
    \widetilde{\boldsymbol{A}}^{(t)}
    =
    \operatorname{softmax}\!\left(
        \frac{
            \boldsymbol{Q}^{(t)}
            \boldsymbol{K}^{(t)\top}
        }{\sqrt{d_h}}
        +
        \boldsymbol{M}
        +
        \boldsymbol{R}
    \right),
    \qquad
    \widetilde{\boldsymbol{O}}^{(t)}
    =
    \widetilde{\boldsymbol{A}}^{(t)}
    \boldsymbol{V}^{(t)} .
\label{eq:wise_attention}
\end{equation}
Thus, \texttt{WISE} reuses only the routing support: $\boldsymbol{H}^{(t)}$, $\boldsymbol{Q}^{(t)}$, $\boldsymbol{K}^{(t)}$, $\boldsymbol{V}^{(t)}$, and the attention weights within the working set remain dynamic throughout recurrence.
It therefore fixes \emph{where} attention may route without freezing \emph{how} the retained information is weighted and transformed.

During autoregressive decoding, the discovered support for the final prompt query block determines which historical KV states are retained for later recurrent steps, while earlier steps requiring unrestricted attention preserve their full KV caches.
Newly generated KV states are incorporated as decoding proceeds.
This compact-cache strategy reduces persistent storage and repeated memory accesses while preserving dynamic within-support attention computation.

\subsection{Computational Complexity}
\label{sec:wise_complexity}

Let $m$ denote the average number of key positions represented by the active key blocks for each query during the reuse phase, let $\rho=m/n$ denote the corresponding normalized working-set density, and let $\alpha=t_{\mathrm d}/T$ denote the fraction of recurrent steps devoted to discovery.
Full attention incurs $O(n^2d)$ routing cost at every recurrent step, whereas block-structured reuse reduces the late-stage cost to $O(nmd)$.
Ignoring block-boundary and kernel overhead, the relative attention-routing cost is therefore
\begin{equation}
    \frac{
        C_{\mathrm{WISE}}^{\mathrm{attn}}
    }{
        C_{\mathrm{full}}^{\mathrm{attn}}
    }
    \approx
    \alpha+(1-\alpha)\rho,
    \qquad
    \Delta C^{\mathrm{attn}}
    \approx
    (1-\alpha)(1-\rho).
\label{eq:wise_complexity}
\end{equation}
The gain increases as the working set becomes sparse relative to the full context.
Equivalently, the late recurrent routing term changes from $O(n^2d)$ to $O(nmd)$, replacing dependence on the full key set with dependence on the discovered working-set size while preserving recurrent representation refinement.

During decoding, when the compact cache retains $m$ historical positions for a later recurrent step, its KV-cache storage decreases from $O(nd)$ to $O(md)$.
This reduction also lowers repeated KV-cache accesses and can improve GPU residency under long-context workloads.



\section{Experiments}

\subsection{Experimental Setup}
\label{sec:experimental_setup}

\paragraph{Models and benchmarks.}
We evaluate \texttt{WISE} on two recurrent-depth language models: Huginn with recurrent depth $T=32$ \citep{NEURIPS2025_3b01972c} and Ouro-1.4B with $T=4$.
For downstream evaluation, we use GSM8K \citep{cobbe2021trainingverifierssolvemath} and ARC-Challenge, covering mathematical reasoning and multiple-choice question answering.
GSM8K uses 128 fixed test examples on Huginn and the full 1,319-example test split on Ouro, while ARC-Challenge uses the full 1,172-example test split on both backbones.
For long-context evaluation, we use the complete 13-task RULER benchmark at 16K, 32K, and 64K context lengths on Ouro, with 64 examples per task and length.
Mechanism diagnostics use fixed 30-example cohorts from HotpotQA \citep{yang2018hotpotqa} and GSM8K.

\textbf{\texttt{WISE} configuration.}
We use backbone-specific discovery configurations without benchmark-specific tuning.
For Huginn, we set block size $B=32$, discovery depth $t_{\mathrm d}=12$, and mass threshold $\eta=0.95$, constructing the working set directly from step-12 attention mass.
For Ouro, we use $B=32$, $t_{\mathrm d}=2$, and $\eta=0.97$, estimating step-2 attention mass from eight sampled queries per block for efficient online discovery.
In both models, the discovered support is reused during later recurrent steps, while hidden states, queries, keys, values, and within-support attention weights remain dynamic.
For Ouro, compact KV-cache management additionally exploits the discovered support during decoding.

\paragraph{Evaluation.}
On GSM8K and ARC-Challenge, we compare \texttt{WISE} against Full, Truncate, STATIC-95, MASSMATCHED, SIZEMATCHED, and Freeze-A.
For RULER, the main Full, WISE, and SIZEMATCHED comparisons use the same 64 examples per task and context length.
We report GSM8K and ARC-Challenge accuracy, official RULER task scores, block density, and retained future Full-attention mass, using matched inputs and paired comparisons across methods.
For systems efficiency, we measure KV-cache memory consumption, decode-only throughput, and end-to-end serving throughput on an NVIDIA RTX A6000.

\subsection{Mechanism Discovery}
\label{sec:mechanism}

\paragraph{Routing Stabilizes Before Representation Refinement.}
We quantify how different components of recurrent computation stabilize across depth.
For diagnostic token support, we track up to the 16 highest-mass causally admissible keys per query and define convergence when consecutive supports reach Jaccard similarity of at least $0.90$ for three successive transitions.
For attention distributions, hidden states, and attention outputs, we use Jensen--Shannon divergence or relative $\ell_2$ change, respectively, with a matched three-transition criterion calibrated to $0.05$ times each example's early-trajectory scale.
As shown in Figure~\ref{fig:wise_motivation}\subref{fig:mechanism}, routing-related quantities consistently stabilize earlier than representation-related quantities across both HotpotQA and GSM8K, although the precise ordering within routing is task-dependent.
This separation suggests that recurrent inference enters a regime in which the model has largely settled on where to retrieve information while representation refinement is still ongoing.
We further verify that the deployed $B=32$, $\eta=0.95$ block-structured working set enters a high-stability regime near $t_{\mathrm d}=12$, with mean adjacent rolling-set Jaccard similarity of $0.9674$ on the fixed 4K HotpotQA audit cohort.

\paragraph{Robustness and cross-backbone validation.}
The routing--representation separation remains robust across consistently varied convergence criteria on Huginn.
For Ouro with only $T=4$ recurrent steps, we instead examine whether early discovered supports remain predictive of later attention.
Using the step-2 working set, \texttt{WISE} retains $96.40\%$ and $95.79\%$ of future Full-attention mass at steps 3 and 4, respectively, compared with $92.41\%$ and $91.86\%$ for exactly size-matched early-static supports.
These results support the effectiveness of recurrent working-set discovery even with shallow recurrent depth.

\subsection{Causal Validation of Support-Only Reuse}
\label{sec:causal_validation}

We next test whether the reusable component is the routing support rather than recurrent refinement or the full attention distribution.
\textsc{Truncate} stops at the discovery depth; \textsc{Static-95}, \textsc{MassMatched}, and \textsc{SizeMatched} select support from the first recurrent step, controlling for the support threshold, retained mass, and support cardinality, respectively.
\textsc{Freeze-A} reuses the same support as \texttt{WISE} but freezes its within-support attention distribution at discovery.
Table~\ref{tab:causal_controls} shows that \texttt{WISE} retains substantially more GSM8K accuracy than these stronger or earlier restrictions, while remaining close to Full attention.
On ARC-Challenge, Full and \texttt{WISE} are more closely matched, and early-static methods are also competitive, indicating that the benefit of delayed support discovery depends on the task.

\begin{table}[t]
\centering
\small
\caption{\textbf{Controlled interventions across recurrent backbones.}
Accuracy (\%) on GSM8K and ARC-Challenge.
GSM8K uses 128 examples for Huginn and 1,319 for Ouro;
ARC-Challenge uses the full 1,172-example test split for both models.
All methods preserve the same evaluation protocol within each backbone.}
\label{tab:causal_controls}
\setlength{\tabcolsep}{5pt}
\begin{tabular}{lcccc}
\toprule
& \multicolumn{2}{c}{\textbf{Huginn}}
& \multicolumn{2}{c}{\textbf{Ouro-1.4B}} \\
\cmidrule(lr){2-3}
\cmidrule(lr){4-5}
\textbf{Method}
& \textbf{GSM8K} & \textbf{ARC-C}
& \textbf{GSM8K} & \textbf{ARC-C} \\
\midrule
Full
& \underline{24.22} & \textbf{37.97}
& \textbf{78.77} & \textbf{52.65} \\
Truncate
& 20.31 & \underline{37.71}
& 64.67 & 47.27 \\
Static-95
& 2.34 & 36.86
& 64.29 & 51.96 \\
MassMatched
& 11.72 & 37.29
& 39.12 & 51.28 \\
SizeMatched
& 18.75 & 37.29
& 48.82 & 51.37 \\
Freeze-A
& 21.09 & 37.63
& 56.56 & 46.08 \\
\textbf{\texttt{WISE}}
& \textbf{27.34} & 37.37
& \underline{77.26} & \underline{52.30} \\
\bottomrule
\end{tabular}
\end{table}

\subsection{Practical GPU Efficiency}
\label{sec:gpu_efficiency}

\paragraph{Block-sparse attention acceleration.}
Native FlashAttention is highly optimized for dense tiled execution but cannot directly exploit the input-dependent block-sparse schedules discovered by \texttt{WISE}.
We therefore implement an exact-$B=32$ Triton kernel \citep{tillet2019triton} for Huginn on NVIDIA Ampere GPUs, using a reusable GPU-resident sparse schedule to skip inactive key blocks while fusing query--key scoring, causal masking, online softmax, and value aggregation.
As shown in Table~\ref{tab:gpu_efficiency}, \texttt{WISE} achieves a $1.76\times$ late-reuse attention speedup over native FlashAttention at 4K, while the complete recurrent attention trajectory remains $1.36\times$ faster after including unrestricted early attention.
A matched sparse-kernel comparison further yields a $2.50\times$ speedup from reducing active routing support alone.
These results demonstrate that the discovered sparsity translates into practical GPU acceleration, although attention-only measurements exclude online discovery and non-attention model computation.

\begin{table}[t]
\centering
\footnotesize
\caption{
\textbf{Block-sparse attention acceleration on Huginn at 4K.}
Measurements use one NVIDIA RTX A6000 (BF16, batch 1), with 30 saved support fixtures and 10 balanced timing passes.
Native FlashAttention comparisons include sparse-schedule construction from saved supports, while the matched $B=32$ comparison uses prepared schedules.
Online support discovery and non-attention model computation are excluded.
}
\label{tab:gpu_efficiency}
\setlength{\tabcolsep}{3pt}
\begin{tabularx}{\columnwidth}{@{}Xrrrr@{}}
\toprule
\textbf{Measurement}
& \textbf{Context}
& \textbf{Full}
& \textbf{\texttt{WISE}}
& \textbf{Speedup} \\
\midrule
Late attention, Native FA (ms)
& 4K & 161.382 & 91.820
& \textbf{1.758$\times$} \\
Full $T=32$ attention (ms)
& 4K & 258.787 & 190.973
& \textbf{1.355$\times$} \\
Late attention, matched $B=32$ (ms)
& 4K & 231.130 & 92.317
& \textbf{2.504$\times$} \\
\bottomrule
\end{tabularx}
\end{table}

\paragraph{Compact KV-cache decoding and serving.}
Beyond prefill computation, \texttt{WISE} reduces decoding costs by retaining only the discovered working-set KV states for later recurrent steps while preserving full KV caches for earlier steps requiring unrestricted attention.
We implement compact KV-cache management on Ouro and compare against lossless Packed Full inference with equivalent memory-management optimizations.
As shown in Table~\ref{tab:decode_efficiency}, \texttt{WISE} substantially reduces persistent KV storage at 64K, allowing its compact cache to remain GPU-resident while Full inference requires host-memory offloading.
This eliminates repeated host--device KV transfers during decoding, yielding a $6.46\times$ decode-only throughput improvement and a $2.58\times$ end-to-end serving throughput improvement under the evaluated latency constraints.
These results demonstrate that working-set reuse reduces both redundant attention computation and memory movement, with particularly substantial benefits near GPU memory-capacity limits.

\begin{table}[t]
\centering
\small
\caption{
\textbf{KV-cache compression and serving efficiency on Ouro-1.4B.}
Measurements use an NVIDIA RTX A6000 (BF16), with lossless Packed Full as the baseline.
Persistent KV is averaged over 832 RULER requests per method and includes GPU and host storage.
Decode-only throughput excludes prefill, while end-to-end serving includes discovery, cache management, scheduling, and generation.
Serving throughput is measured on finite open-loop traces under shared p95 latency constraints.
}
\label{tab:decode_efficiency}
\setlength{\tabcolsep}{4.5pt}
\begin{tabular}{lcccc}
\toprule
\textbf{Measurement}
& \textbf{Context}
& \textbf{Full}
& \textbf{\texttt{WISE}}
& \textbf{Improvement} \\
\midrule
Persistent KV (GiB)
& 64K & 47.40 & 27.78
& \textbf{41.4\% less} \\
Decode throughput (tok/s)
& 64K & 1.871 & 12.082
& \textbf{6.456$\times$} \\
End-to-end serving (tok/s)
& 64K & 1.220 & 3.147
& \textbf{2.580$\times$} \\
\bottomrule
\end{tabular}
\end{table}

\subsection{Quality--Efficiency Scaling with Context Length}
\label{sec:context_scaling}

We evaluate long-context quality on the complete 13-task RULER suite, rather than restricting analysis to retrieval tasks.
Table~\ref{tab:ruler_macro} reports the task-equal official score at each context length: \texttt{WISE} consistently exceeds the size-matched early-static baseline, but falls below Full attention, particularly at 16K.
The smaller gap at 64K should be interpreted alongside the declining absolute accuracy of Full and the larger losses on aggregation tasks; it does not establish uniform quality preservation.
Figure~\ref{fig:quality_efficiency_scaling} complements this full-suite evaluation with the eight-task retrieval subset, decode throughput scaling, and the block-size tradeoff.
At 64K, the larger decode gain arises because the compact cache remains GPU-resident while Full requires host-memory offloading.

\begin{table}[t]
\centering
\small
\caption{\textbf{Long-context performance on the full RULER suite (Ouro-1.4B).}
Task-equal mean official RULER scores (\%), retaining task-specific partial credit, across all 13 tasks. Methods use the same 64 examples per task and context length.}
\label{tab:ruler_macro}
\setlength{\tabcolsep}{6pt}
\begin{tabular}{lccc}
\toprule
\textbf{Context} & \textbf{Full} & \textbf{\texttt{WISE}} & \textbf{SizeMatched} \\
\midrule
16K & \textbf{72.40} & 63.17 & 55.42 \\
32K & \textbf{54.80} & 51.64 & 48.91 \\
64K & \textbf{42.79} & 42.08 & 37.18 \\
\bottomrule
\end{tabular}
\end{table}

\begin{figure*}[t]
\centering

\begin{minipage}[c]{0.32\textwidth}
\centering
\includegraphics[
    width=\linewidth,
    keepaspectratio
]{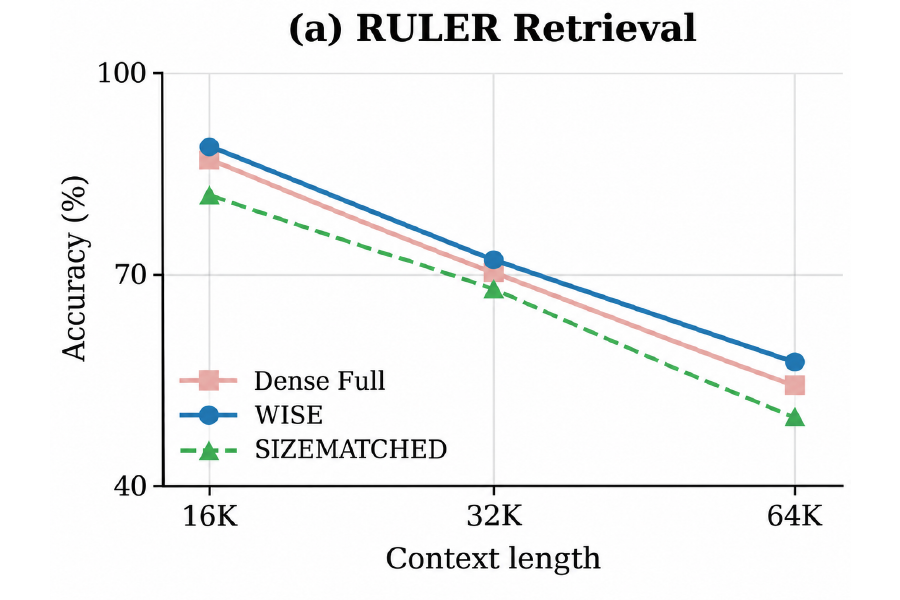}
\phantomsubcaption
\label{fig:scaling_structure}
\end{minipage}
\hfill
\begin{minipage}[c]{0.32\textwidth}
\centering
\includegraphics[
    width=\linewidth,
    keepaspectratio
]{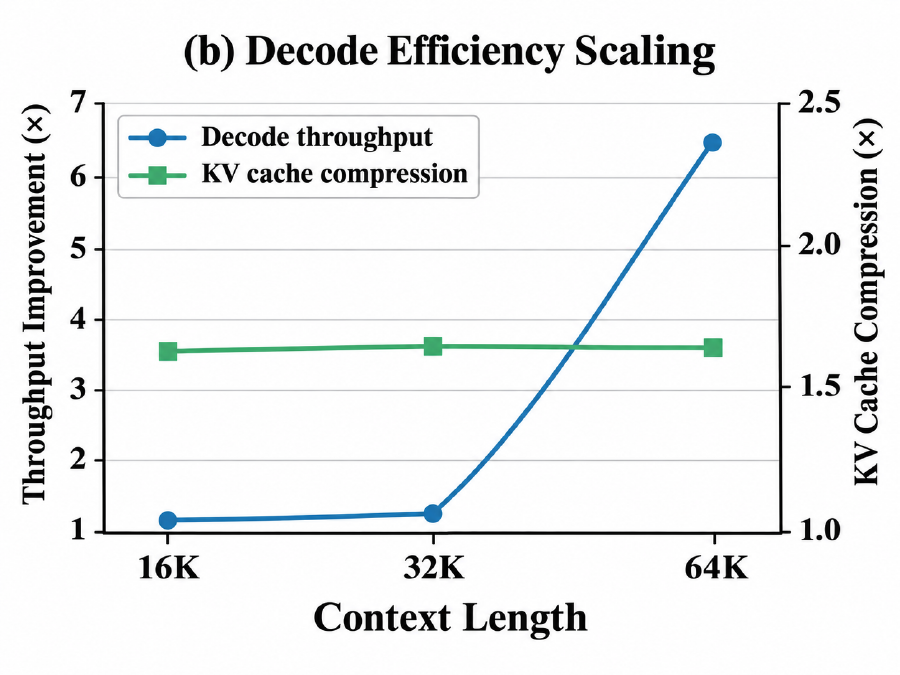}
\phantomsubcaption
\label{fig:scaling_quality_efficiency}
\end{minipage}
\hfill
\begin{minipage}[c]{0.32\textwidth}
\centering
\includegraphics[
    width=\linewidth,
    keepaspectratio
]{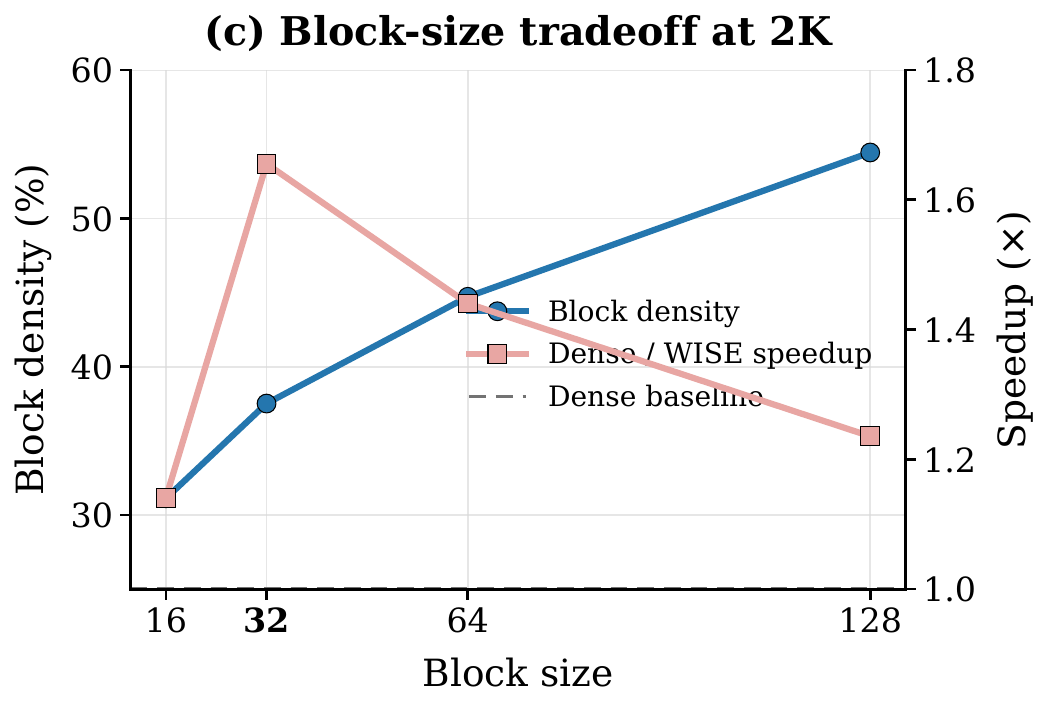}
\phantomsubcaption
\label{fig:block_size_tradeoff}
\end{minipage}

\caption{
\textbf{Quality--efficiency scaling and block-size tradeoffs in \texttt{WISE}.}
\textbf{(a)} Average retrieval accuracy on the eight RULER retrieval tasks for Ouro at 16K--64K, comparing Full, \texttt{WISE}, and \textsc{SizeMatched}.
\textbf{(b)} Decode throughput improvement and KV-cache compression (Full/\texttt{WISE}) across context lengths on Ouro.
The 16K and 32K results are peak GPU-resident decode throughput at each method's best tested batch size; the 64K result is from separate decode-only traces in which Full requires KV offloading while \texttt{WISE} remains resident.
\textbf{(c)} Block-size tradeoff on Huginn at 2K: finer blocks expose greater sparsity, while $B=32$ achieves the highest measured setup-inclusive attention speedup.
}
\label{fig:quality_efficiency_scaling}

\end{figure*}

\textbf{Block-size tradeoff.}
Under equal implementation-tuning budgets, $B=16$ yields the sparsest support, while $B=32$ achieves the highest setup-inclusive attention speedup.
Coarser blocks increase support density without consistently improving execution efficiency.
No tested granularity consistently dominates downstream quality, so we select $B=32$ as the practical balance between routing sparsity and efficient GPU execution.

\section{Conclusion}

We identify a separation in recurrent-depth language-model inference: attention routing stabilizes substantially earlier than representation refinement.
Motivated by this observation, we introduce \texttt{WISE}, a training-free method that reuses early-discovered sparse routing support while preserving recurrent depth and dynamic within-support computation.
Controlled evaluations across Huginn and Ouro demonstrate the value of recurrent working-set discovery over early-static support selection and more restrictive reuse strategies, while revealing task-dependent limitations in quality preservation.
Beyond theoretical sparsity, \texttt{WISE} translates reusable routing structure into practical GPU efficiency through block-sparse attention acceleration and compact KV-cache decoding.
At 64K context, \texttt{WISE} achieves up to a $6.46\times$ decode-only throughput improvement and a $2.58\times$ end-to-end serving throughput improvement under the evaluated GPU memory and latency constraints.
Together, these findings establish working-set reuse as a practical approach to reducing both redundant attention computation and memory movement during recurrent inference.









\bibliographystyle{iclr2027_conference}
\bibliography{main}

@inproceedings{ICLR2024_98ed250b,
 author = {Dao, Tri},
 booktitle = {International Conference on Learning Representations},
 editor = {B. Kim and Y. Yue and S. Chaudhuri and K. Fragkiadaki and M. Khan and Y. Sun},
 pages = {35549--35562},
 title = {FlashAttention-2: Faster Attention with Better Parallelism and Work Partitioning},
 url = {https://proceedings.iclr.cc/paper_files/paper/2024/file/98ed250b203d1ac6b24bbcf263e3d4a7-Paper-Conference.pdf},
 volume = {2024},
 year = {2024}
}

@inproceedings{NEURIPS2024_5dfbe6f5,
 author = {Jiang, Huiqiang and Li, Yucheng and Zhang, Chengruidong and Wu, Qianhui and Luo, Xufang and Ahn, Surin and Han, Zhenhua and Abdi, Amir H. and Li, Dongsheng and Lin, Chin-Yew and Yang, Yuqing and Qiu, Lili},
 booktitle = {Advances in Neural Information Processing Systems},
 doi = {10.52202/079017-1663},
 editor = {A. Globerson and L. Mackey and D. Belgrave and A. Fan and U. Paquet and J. Tomczak and C. Zhang},
 pages = {52481--52515},
 publisher = {Curran Associates, Inc.},
 title = {MInference 1.0: Accelerating Pre-filling for Long-Context LLMs via Dynamic Sparse Attention},
 url = {https://proceedings.neurips.cc/paper_files/paper/2024/file/5dfbe6f5671e82c76841ba687a8a9ecb-Paper-Conference.pdf},
 volume = {37},
 year = {2024}
}

@inproceedings{NEURIPS2023_6ceefa7b,
 author = {Zhang, Zhenyu and Sheng, Ying and Zhou, Tianyi and Chen, Tianlong and Zheng, Lianmin and Cai, Ruisi and Song, Zhao and Tian, Yuandong and R\'{e}, Christopher and Barrett, Clark and Wang, Zhangyang "Atlas" and Chen, Beidi},
 booktitle = {Advances in Neural Information Processing Systems},
 doi = {10.52202/075280-1506},
 editor = {A. Oh and T. Naumann and A. Globerson and K. Saenko and M. Hardt and S. Levine},
 pages = {34661--34710},
 publisher = {Curran Associates, Inc.},
 title = {H2O: Heavy-Hitter Oracle for Efficient Generative Inference of Large Language Models},
 url = {https://proceedings.neurips.cc/paper_files/paper/2023/file/6ceefa7b15572587b78ecfcebb2827f8-Paper-Conference.pdf},
 volume = {36},
 year = {2023}
}

@inproceedings{NEURIPS2022_8bb0d291,
 author = {Kojima, Takeshi and Gu, Shixiang (Shane) and Reid, Machel and Matsuo, Yutaka and Iwasawa, Yusuke},
 booktitle = {Advances in Neural Information Processing Systems},
 doi = {10.52202/068431-1613},
 editor = {S. Koyejo and S. Mohamed and A. Agarwal and D. Belgrave and K. Cho and A. Oh},
 pages = {22199--22213},
 publisher = {Curran Associates, Inc.},
 title = {Large Language Models are Zero-Shot Reasoners},
 url = {https://proceedings.neurips.cc/paper_files/paper/2022/file/8bb0d291acd4acf06ef112099c16f326-Paper-Conference.pdf},
 volume = {35},
 year = {2022}
}

@inproceedings{rodkin-etal-2026-beyond,
    title = "Beyond Memorization: Extending Reasoning Depth with Recurrence, Memory and Test-Time Compute Scaling",
    author = "Rodkin, Ivan  and
      Orel, Daniil  and
      Smirnov, Konstantin  and
      Bolatov, Arman  and
      Elbouardi, Bilal  and
      Hassan, Besher  and
      Kuratov, Yuri  and
      Bulatov, Aydar  and
      Nakov, Preslav  and
      Baldwin, Timothy  and
      Shelmanov, Artem  and
      Burtsev, Mikhail",
    editor = "Liakata, Maria  and
      Moreira, Viviane P.  and
      Zhang, Jiajun  and
      Jurgens, David",
    booktitle = "Findings of the {A}ssociation for {C}omputational {L}inguistics: {ACL} 2026",
    month = jul,
    year = "2026",
    address = "San Diego, California, United States",
    publisher = "Association for Computational Linguistics",
    url = "https://aclanthology.org/2026.findings-acl.2103/",
    doi = "10.18653/v1/2026.findings-acl.2103",
    pages = "42385--42404",
    ISBN = "979-8-89176-395-1"
}

@inproceedings{elhoushi-etal-2024-layerskip,
    title = "{L}ayer{S}kip: Enabling Early Exit Inference and Self-Speculative Decoding",
    author = "Elhoushi, Mostafa  and
      Shrivastava, Akshat  and
      Liskovich, Diana  and
      Hosmer, Basil  and
      Wasti, Bram  and
      Lai, Liangzhen  and
      Mahmoud, Anas  and
      Acun, Bilge  and
      Agarwal, Saurabh  and
      Roman, Ahmed  and
      Aly, Ahmed  and
      Chen, Beidi  and
      Wu, Carole-Jean",
    editor = "Ku, Lun-Wei  and
      Martins, Andre  and
      Srikumar, Vivek",
    booktitle = "Proceedings of the 62nd Annual Meeting of the Association for Computational Linguistics (Volume 1: Long Papers)",
    month = aug,
    year = "2024",
    address = "Bangkok, Thailand",
    publisher = "Association for Computational Linguistics",
    url = "https://aclanthology.org/2024.acl-long.681/",
    doi = "10.18653/v1/2024.acl-long.681",
    pages = "12622--12642"
}

@misc{cobbe2021trainingverifierssolvemath,
      title={Training Verifiers to Solve Math Word Problems}, 
      author={Karl Cobbe and Vineet Kosaraju and Mohammad Bavarian and Mark Chen and Heewoo Jun and Lukasz Kaiser and Matthias Plappert and Jerry Tworek and Jacob Hilton and Reiichiro Nakano and Christopher Hesse and John Schulman},
      year={2021},
      eprint={2110.14168},
      archivePrefix={arXiv},
      primaryClass={cs.LG},
      url={https://arxiv.org/abs/2110.14168}, 
}

@inproceedings{NEURIPS2025_3b01972c,
 author = {Geiping, Jonas and McLeish, Sean and Jain, Neel and Kirchenbauer, John and Singh, Siddharth and Bartoldson, Brian and Kailkhura, Bhavya and Bhatele, Abhinav and Goldstein, Tom},
 booktitle = {Advances in Neural Information Processing Systems},
 doi = {10.52202/085713-1380},
 editor = {D. Belgrave and C. Zhang and H. Lin and R. Pascanu and P. Koniusz and M. Ghassemi and N. Chen},
 pages = {41340--41391},
 publisher = {Curran Associates, Inc.},
 title = {Scaling up Test-Time Compute with Latent Reasoning: A Recurrent Depth Approach},
 url = {https://proceedings.neurips.cc/paper_files/paper/2025/file/3b01972cf31e6fa0fe29e4b8b5c2a0a1-Paper-Conference.pdf},
 volume = {38, Main Conference},
 year = {2025}
}

@inproceedings{xu-etal-2025-refreshkv,
    title = "{R}efresh{KV}: Updating Small {KV} Cache During Long-form Generation",
    author = "Xu, Fangyuan  and
      Goyal, Tanya  and
      Choi, Eunsol",
    editor = "Che, Wanxiang  and
      Nabende, Joyce  and
      Shutova, Ekaterina  and
      Pilehvar, Mohammad Taher",
    booktitle = "Proceedings of the 63rd Annual Meeting of the Association for Computational Linguistics (Volume 1: Long Papers)",
    month = jul,
    year = "2025",
    address = "Vienna, Austria",
    publisher = "Association for Computational Linguistics",
    url = "https://aclanthology.org/2025.acl-long.1211/",
    doi = "10.18653/v1/2025.acl-long.1211",
    pages = "24878--24893",
    ISBN = "979-8-89176-251-0"
}

@article{tang2024quest,
  title={Quest: Query-aware sparsity for efficient long-context llm inference},
  author={Tang, Jiaming and Zhao, Yilong and Zhu, Kan and Xiao, Guangxuan and Kasikci, Baris and Han, Song},
  journal={arXiv preprint arXiv:2406.10774},
  year={2024}
}

@inproceedings{yang2018hotpotqa,
  title={HotpotQA: A dataset for diverse, explainable multi-hop question answering},
  author={Yang, Zhilin and Qi, Peng and Zhang, Saizheng and Bengio, Yoshua and Cohen, William and Salakhutdinov, Ruslan and Manning, Christopher D},
  booktitle={Proceedings of the 2018 conference on empirical methods in natural language processing},
  pages={2369--2380},
  year={2018}
}

@inproceedings{tillet2019triton,
  title={Triton: an intermediate language and compiler for tiled neural network computations},
  author={Tillet, Philippe and Kung, Hsiang-Tsung and Cox, David},
  booktitle={Proceedings of the 3rd ACM SIGPLAN International Workshop on Machine Learning and Programming Languages},
  pages={10--19},
  year={2019}
}

@article{deshmukh2025kascade,
  title={Kascade: A practical sparse attention method for long-context llm inference},
  author={Deshmukh, Dhruv and Goyal, Saurabh and Kwatra, Nipun and Ramjee, Ramachandran},
  journal={arXiv preprint arXiv:2512.16391},
  year={2025}
}

@article{dao2022flashattention,
  title={Flashattention: Fast and memory-efficient exact attention with io-awareness},
  author={Dao, Tri and Fu, Dan and Ermon, Stefano and Rudra, Atri and R{\'e}, Christopher},
  journal={Advances in neural information processing systems},
  volume={35},
  pages={16344--16359},
  year={2022}
}

@article{xiao2019sharing,
  title={Sharing attention weights for fast transformer},
  author={Xiao, Tong and Li, Yinqiao and Zhu, Jingbo and Yu, Zhengtao and Liu, Tongran},
  journal={arXiv preprint arXiv:1906.11024},
  year={2019}
}

@article{zhou2020bert,
  title={Bert loses patience: Fast and robust inference with early exit},
  author={Zhou, Wangchunshu and Xu, Canwen and Ge, Tao and McAuley, Julian and Xu, Ke and Wei, Furu},
  journal={Advances in Neural Information Processing Systems},
  volume={33},
  pages={18330--18341},
  year={2020}
}

@article{banino2021pondernet,
  title={Pondernet: Learning to ponder},
  author={Banino, Andrea and Balaguer, Jan and Blundell, Charles},
  journal={arXiv preprint arXiv:2107.05407},
  year={2021}
}

@inproceedings{xin2020deebert,
  title={DeeBERT: Dynamic early exiting for accelerating BERT inference},
  author={Xin, Ji and Tang, Raphael and Lee, Jaejun and Yu, Yaoliang and Lin, Jimmy},
  booktitle={Proceedings of the 58th annual meeting of the association for computational linguistics},
  pages={2246--2251},
  year={2020}
}

@article{elbayad2019depth,
  title={Depth-adaptive transformer},
  author={Elbayad, Maha and Gu, Jiatao and Grave, Edouard and Auli, Michael},
  journal={arXiv preprint arXiv:1910.10073},
  year={2019}
}

@article{kohli2026loop,
  title={Loop, think, \& generalize: Implicit reasoning in recurrent-depth transformers},
  author={Kohli, Harsh and Parthasarathy, Srinivasan and Sun, Huan and Yao, Yuekun},
  journal={arXiv preprint arXiv:2604.07822},
  year={2026}
}

@inproceedings{zhang2026modr,
  title={Modr: Mixture-of-depth-recurrent transformers for test-time reasoning},
  author={Zhang, Xiaojing and Wu, Haifeng and He, Gang and Shen, Jiyang and Lyu, Bochen and Zhu, Zhanxing},
  booktitle={International Conference on Learning Representations},
  volume={2026},
  pages={85952--85975},
  year={2026}
}

@inproceedings{saunshi2025reasoning,
  title={Reasoning with latent thoughts: On the power of looped transformers},
  author={Saunshi, Nikunj and Dikkala, Nishanth and Li, Zhiyuan and Kumar, Sanjiv and J Reddi, Sashank},
  booktitle={International Conference on Learning Representations},
  volume={2025},
  pages={14855--14881},
  year={2025}
}

@article{li2024snapkv,
  title={Snapkv: Llm knows what you are looking for before generation},
  author={Li, Yuhong and Huang, Yingbing and Yang, Bowen and Venkitesh, Bharat and Locatelli, Acyr and Ye, Hanchen and Cai, Tianle and Lewis, Patrick and Chen, Deming},
  journal={Advances in Neural Information Processing Systems},
  volume={37},
  pages={22947--22970},
  year={2024}
}

@article{raposo2024mixture,
  title={Mixture-of-depths: Dynamically allocating compute in transformer-based language models},
  author={Raposo, David and Ritter, Sam and Richards, Blake and Lillicrap, Timothy and Humphreys, Peter Conway and Santoro, Adam},
  journal={arXiv preprint arXiv:2404.02258},
  year={2024}
}

@article{hao2024training,
  title={Training large language models to reason in a continuous latent space},
  author={Hao, Shibo and Sukhbaatar, Sainbayar and Su, DiJia and Li, Xian and Hu, Zhiting and Weston, Jason and Tian, Yuandong},
  journal={arXiv preprint arXiv:2412.06769},
  year={2024}
}

@article{graves2016adaptive,
  title={Adaptive computation time for recurrent neural networks},
  author={Graves, Alex},
  journal={arXiv preprint arXiv:1603.08983},
  year={2016}
}

@inproceedings{giannou2023looped,
  title={Looped transformers as programmable computers},
  author={Giannou, Angeliki and Rajput, Shashank and Sohn, Jy-yong and Lee, Kangwook and Lee, Jason D and Papailiopoulos, Dimitris},
  booktitle={International Conference on Machine Learning},
  pages={11398--11442},
  year={2023},
  organization={PMLR}
}

@article{dehghani2018universal,
  title={Universal transformers},
  author={Dehghani, Mostafa and Gouws, Stephan and Vinyals, Oriol and Uszkoreit, Jakob and Kaiser, {\L}ukasz},
  journal={arXiv preprint arXiv:1807.03819},
  year={2018}
}

@inproceedings{yang2024looped,
  title={Looped transformers are better at learning learning algorithms},
  author={Yang, Liu and Lee, Kangwook and Nowak, Robert and Papailiopoulos, Dimitris},
  booktitle={International conference on learning representations},
  volume={2024},
  pages={42195--42214},
  year={2024}
}

\newpage
\appendix








\end{document}